\documentclass[letterpaper, 10pt, conference]{ieeeconf}      

\IEEEoverridecommandlockouts                              
\usepackage{url}
\usepackage{graphics} 
\usepackage{epsfig} 
\usepackage{times} 
\usepackage{amsmath} 
\usepackage{amssymb}  

\usepackage{amsmath} 
\usepackage{bm} 
\usepackage{color} 

\usepackage{tikz}
\let\labelindent\relax
\usepackage{enumitem}
\setlist{nolistsep} 
\usepackage{verbatim}
\usetikzlibrary{arrows,shapes,positioning,shadows,trees}

\tikzset{
	basic/.style  = {draw, text width=3cm, rectangle},
	root/.style   = {basic, rounded corners=2pt, thin, align=center,
		fill=green!0},
	level 2/.style = {basic, rounded corners=3pt, thin,align=center, fill=green!20,
		text width=9em},
	level 3/.style = {basic, ultra thin, align=center, fill=yellow!10,text width=8em}
}

\usepackage{amsmath}
\usepackage{amssymb}
\usepackage{pifont}
\usepackage{cite}
\usepackage{tikz}
\usepackage{tikz-cd}
\usetikzlibrary{shapes,arrows,fit,backgrounds,positioning}
\usetikzlibrary{decorations.pathmorphing} 
\usetikzlibrary{matrix} 
\usetikzlibrary{arrows} 
\usetikzlibrary{calc} 
\usepackage{verbatim}
\usepackage[table]{colortbl}
\usepackage{accents}
\usepackage{algorithm}
\usepackage{caption}
\usepackage[noend]{algpseudocode}

\DeclareMathOperator*{\argmin}{arg\,min}

\newtheorem{definition}{Definition}[section]
\newtheorem{remark}{Remark}[section]
\newtheorem{assumption}{Assumption}[section]
\newtheorem{theorem}{Theorem}[section]
\newtheorem{lemma}{Lemma}[section]
\newtheorem{property}{Property}[section]
\newtheorem{proposition}{Proposition}[section]
\newtheorem{corollary}{Corollary}[section]
\newtheorem{example}{Example}[section]

\usepackage{import}
\usepackage{xifthen}
\usepackage{pdfpages}
\usepackage{transparent}

\usepackage{setspace}

\usepackage{tablefootnote}

\makeatletter
\let\NAT@parse\undefined
\makeatother
\usepackage{hyperref}  

\newcommand\Set[1]{\mathbb{#1}} 

\newcommand\ko[0]{{\mathcal{K}}}

\newcommand{\tsgn}[1]{{#1}}

\newcommand\set[1]{\mathbb{#1}}

\usepackage{amsfonts}
  \usepackage{pgfplots}
  \pgfplotsset{compat=1.10}
  \usetikzlibrary{plotmarks}
  \usetikzlibrary{arrows.meta}
  \usepgfplotslibrary{patchplots}
  \usepackage{grffile}
  \usepackage{amsmath}
\usetikzlibrary{trees}
\usetikzlibrary{shadings}
\usetikzlibrary{backgrounds}

\title{\LARGE \bf 
Learning the Koopman Eigendecomposition: A Diffeomorphic Approach
}

\author{Petar Bevanda, Johannes Kirmayr, Stefan Sosnowski, Sandra Hirche
	\thanks{This work was supported by European Union's Horizon 2020 research and innovation programme under grant agreement no. 871295 "SeaClear" ({SEarch, identificAtion and Collection of marine Litter with Autonomous Robots}).}
	\thanks{All authors are members of the Chair of Information-oriented Control, Department of Electrical and Computer Engineering, Technical University of Munich, D-80333 Munich, Germany,
		{\tt\small \{petar.bevanda,  johannes.kirmayr, sosnowski, hirche\}@tum.de}.}%
}

\begin{document}

\maketitle
\thispagestyle{empty}
\pagestyle{empty}

\begin{abstract}
We present a novel data-driven approach for learning linear representations of a class of stable nonlinear systems using Koopman eigenfunctions. Utilizing the spectral equivalence of topologically conjugate systems, we construct Koopman eigenfunctions corresponding to the nonlinear system to form linear predictors of nonlinear systems. The conjugacy map between a nonlinear system and its Jacobian linearization is learned via a diffeomorphic neural network. The latter allows for a well-defined, supervised learning problem formulation. Given the learner is diffeomorphic per construction, our learned model is asymptotically stable regardless of the representation accuracy. 
The universality of the diffeomorphic learner leads to the universal approximation ability for Koopman eigenfunctions - admitting suitable expressivity.
The efficacy of our approach is demonstrated in simulations.
\end{abstract}
\vspace{0.5em}
\section{INTRODUCTION}\label{secI}
{For complex nonlinear systems, models based on first-principles often do not fully resemble the true system due to unmodeled phenomena. To better deal with the aforementioned, flexible machine learning techniques are employed (e.g. neural networks or Gaussian processes) for prediction \cite{Nelles2001,Kocijan2005}, model-based control \cite{Umlauft2018} and analysis \cite{Berkenkamp2016a,Lederer2019a}.
Although classical nonlinear system representations enjoy incredible success, multi-step prediction, analysis and optimization-based control are substantially more challenging than that of their linear analogues.
Inspired by the infinite-dimensional but \textit{linear} Koopman operator - named after B.O. Koopman's seminal work \cite{Koopman1931} - a rise of interest for global linearizations is observed in various research fields.
Trading infinite-dimensionality for linearity enables the use of efficient linear techniques for nonlinear systems - leading to more challenging identification but efficacious prediction, analysis and control \cite{Bevanda2021}.
The challenge of identification involves "lifting" the original system state to suitable higher-dimensional coordinates that represent the linear Koopman operator (generator) in a finite-dimensional form.}

A dominant train of thought assumes a predefined library of functions approximating the operator - akin to the well-known extended DMD (EDMD) \cite{Williams2015a}. However, a ``good" library of functions representing the Koopman operator (generator) should be both dynamically closed and relevant for reconstructing the original state evolution. Hence, \textit{apriori} access to a suitable function library is a strong assumption - commonly leading to only locally accurate models and no practicable learning guarantees. 
Other approaches leverage the expressive power of neural networks or kernel methods to learn a suitable library of functions \cite{Li2017a,Lian2019} but often lack theoretical justification.

However, the (generalized) eigenfunctions of the operator are dynamically closed coordinates by definition. Thus, to learn a Koopman operator (generator) representation for long-term accurate prediction, it is vital  to construct genuine eigenfunction-coordinates. 
Nevertheless, very few works consider direct learning of genuine Koopman eigenfunctions for linear prediction models, such as Korda et al. \cite{Korda2020b}. Although assuming no spectral or feature knowledge, its performance is dependent on trajectory data, a specific choice of eigenfunction lattices and basis functions for interpolation. Similarly to our own, the work of Folkestad et al. \cite{Folkestad2019} proposes to learn a conjugacy between the nonlinear dynamics and its Jacobian linearization to construct Koopman operator eigenfunctions for linear prediction. Crucially, however, the aforementioned method solves a markedly underdetermined learning problem while employing heuristics that are not theoretically justified.
Furthermore, although considering stable systems, it provides no stability guarantees that would allow for safety and physical consistency even for unseen states \cite{Tesfazgi2021}.

{This paper presents a novel data-driven approach for learning Koopman operator generator eigenfunctions for prediction.
To construct Koopman eigenfunctions, we learn a diffeomorphism between a hyperbolic nonlinear system and its linearization using a {diffeomorphic} neural network. The latter, together with explicit training targets, leads to a well-defined supervised learning problem.
The learner's universal approximation capability of diffeomorphisms transfers to that of the nonlinear system's Koopman eigenfunctions - allowing for sufficient expressivity.
Additionally, our framework also ensures {safety} in the sense of guaranteeing global asymptotic stability of the Koopman operator dynamical model - regardless of the representation accuracy.
The superior performance of our approach - also compared to existing techniques - is demonstrated in simulation examples.}

This paper is structured as follows: After the problem setup and introduction of required system-theoretical results in Sec. \ref{sec:KOthry} and \ref{sec:sysTHRY}, we propose a novel data-driven framework - \textit{KoopmanEigenFlows} - for constructing genuine Koopman operator eigenfunctions in Sec. \ref{sec:KEFlows}. Thereafter in Sec. \ref{sec:KEFDMD}, we propose a linear representation of the nonlinear systems with our algorithm \textbf{K}oopman\textbf{E}igen\textbf{F}low \textbf{M}ode \textbf{D}ecomposition (KEFMD). It is followed by numerical evaluation in Sec. \ref{sec:simres} and a conclusion.
\section{Problem Formulation}\label{sec:KOthry}
 {
 Consider a partially known, continuous-time nonlinear system \footnote{\textbf{Notation:}
		Lower/upper case bold symbols $\boldsymbol{x}$/$\boldsymbol{X}$ denote vectors/matrices.
		Symbols $\mathbb{N}/\mathbb{ R }/\mathbb{C}$ denote sets of natural/real/complex numbers while $\mathbb{N}_{0}$ denotes all natural numbers with zero and $\mathbb{R}_{+,0}/\mathbb{R}_{+}$ all positive reals with/without zero. 
		Function spaces with a specific integrability/smoothness order are denoted as $L^{}$/$C^{}$ with the order (class) specified in their exponent. The Jacobian matrix of map $\bm{h}$ evaluated at $\bm{x}$ is denoted as $\bm{J}_{\bm{h}}(\bm{x})$.

		A flow induced by a vector field $\dot{x}=f(x)$ is denoted as $F^{t}(x)$ with its associated family of composition (Koopman) operators $\{\ko^{t}_{f}\}_{t \in \mathbb{R}_{+,0}}$. The $L^p$-norm on a set $\set{X}$ is denoted as $\|\!\cdot\!\|_{p, \set{X}}$.}
	\begin{equation}\label{sysKdiffeo}
	\dot{\bm{x}}=\bm{f}(\bm{x})=\bm{A}\bm{x}+\bm{r}(\bm{x})
	\end{equation}
 with continuous states on a compact set $\bm{x} \in \set{X} \subset \mathbb{R}^{d}$ {containing the origin}, consisting of a known $\bm{A} \in \set{R}^{d \times d}$ and an unknown $\bm{r}: \set{X} \mapsto \set{R}^d$.
\begin{assumption}\label{ass:sysCLS}
We assume that the origin of \eqref{sysKdiffeo} is globally exponentially stable.
\end{assumption}
The above system class includes dynamical systems representing motion as well as various dissipative Lagrangian systems.
 }
Due to their continuous-time nature, the dynamics are fully described by the forward-complete flow map~\cite{Bittracher2015} of \eqref{sysKdiffeo} given by
\begin{equation}
\label{flow}
 \bm{x}(t_0) \equiv \bm{x}_0, \quad \bm{F}^{t}(\bm{x}_0):= \bm{x}_0+\int_{t_{0}}^{t_{0}+t} \bm{f}(\bm{x}(\tau)) d \tau,
\end{equation}
which has a unique solution on $[0,+\infty)$ from the initial
condition $\bm{x}$ at $t = 0$ due to stability of the isolated attractor~\cite{Angeli1999}. This flow map naturally induces the associated Koopman operator semigroup as defined in the following.
\begin{definition}\label{def:Koop}
	The semigroup of Koopman operators $\{{\mathcal{K}}^{t}\}_{t \in \mathbb{R}_{+,0}}\!:\! C(\Set{X}) \!\mapsto\! C(\Set{X})$ for the flow \eqref{flow} acts on a scalar observable function ${h} \!\in\! C(\Set{X})$ on the state space $\Set{X}$ through ${\mathcal{K}^{t}_{\bm{f}}} {{h}} ={{h}}\circ{\bm{F}^{t}}$.
\end{definition}
\begin{definition}[\!\cite{Lasota1994}]\label{def:generator}
	The operator ${\mathcal{G}_{\mathcal{K}}}$, is the infinitesimal generator
	\begin{equation}
	\mathcal{G}_{{\mathcal{K}}} {h}=\lim _{t \rightarrow 0^{+}} \frac{{\mathcal{K}}^{t} {h}-{h}}{t} = \frac{d}{dt}{h},
	\end{equation}
	of the time-$t$ indexed semigroup of Koopman operators $\{{\mathcal{K}}^{t}\}_{t \in \mathbb{R}_{+,0}}$.
\end{definition}
The natural linearly evolving coordinates are the eigenfunctions of evolution operators.
\begin{definition}\label{eigF}
	An observable $\phi \in  C(\Set{X})$ is an \textit{eigenfunction} if it satisfies
	\begin{equation}\label{KGEVE}
	[\mathcal{G}_{\mathcal{K}} \phi](\boldsymbol{x}) = \dot{\phi} \left(\boldsymbol{x}\right) = \lambda \phi (\boldsymbol{x}),
	\end{equation}
	associated with the \textit{eigenvalue} $\lambda \in \mathbb{C}$.
\end{definition}
\begin{property}\label{prop:linCoor}
	Since $\mathcal{G}_{\mathcal{K}}$ is the infinitesimal generator of the semigroup of Koopman operators $\{{\mathcal{K}}^{t}\}_{t \in \mathbb{R}_{+,0}}$, the following is also satisfied 
	\begin{equation}\label{eq:eigF}
	[\mathcal{K}^{t}_{\boldsymbol{f}} \phi](\boldsymbol{x}) = \phi\left( \boldsymbol{F}^{t} (\boldsymbol{x}) \right) = {e}^{\lambda t} \phi \left(\boldsymbol{x} \right),
	\end{equation}
	along the vector field's flow.
\end{property}
{Due to Assumption~\ref{ass:sysCLS}, 
the Koopman operator generator has a pure point spectrum for the dynamics \eqref{sysKdiffeo} \cite{Mauroy2016b}. 
Thus, for each observable $\bm{h}$, there exists a sequence $\bm{v}_j(h)\in\Set{C}$ of mode weights, such that action of the Koopman generator is represented through the following decomposition 
\begin{equation}\label{evoOF}
\dot{\bm{h}} =\mathcal{G}_{\ko} \bm{h} =\sum_{j=1}^{\infty} \bm{v}_{j}(\bm{h})\left(\mathcal{G}_{\ko} {\phi}_{j}\right)=\sum_{j=1}^{\infty} \bm{v}_{j}(\bm{h}) \lambda_{j} {\phi}_{j}.
\end{equation}}
Given the existence of the decomposition (\ref{evoOF}), we are interested in learning a model of the following form
\begin{subequations}\label{eq:LTI}
\begin{align}
\bm{z}_0 &=\bm{\phi}(\bm{x}(0)), \label{eq:cLTI:1}\\
\dot{\bm{z}} &=\bm{\Lambda} \bm{z}, \label{eq:cLTI:2}\\
\bm{x} &=\bm{V} \bm{z}, \label{eq:cLTI:3}
\end{align}
\end{subequations}
where $\bm{\phi}=[\phi_1,\cdots,\phi_D]^{\top}$ are the finite-dimensional eigenfunction coordinates, $\bm{\Lambda} \in \set{R}^{D \times D}$ and $\bm{V} \in \set{R}^{d \times D}$.
We consider the full-state observable to be the output of interest $\bm{h}(\bm{x})=\operatorname{id}(\bm{x})$ in (\ref{eq:cLTI:3}). The goal is to trade the nonlinearity of a $d$-dimensional ODE (\ref{sysKdiffeo}) for a nonlinear ``lift" (\ref{eq:cLTI:1}) of the initial condition $\bm{x}(0)$ to higher dimension ($D \gg d$) leading to a linearly evolving model $\dot{\bm{x}}=\bm{V}\bm{\Lambda}\bm{\phi}(\bm{x})$ with a closed form flow ${\bm{x}(t)}=\bm{V}e^{\bm{\Lambda}t}\bm{\phi}(\bm{x}(0))$. In general, the output of interest (\ref{eq:cLTI:3}) can be any other observable function $\bm{h}(\bm{x})$ as well.

\vspace{-0.2em}
\section{Modeling via Equivalence Relations}\label{sec:sysTHRY}
To reliably construct Koopman eigenfunctions, we utilize equivalence relations between the nonlinear system \eqref{sysKdiffeo} and its linearization around the origin. To utilize the aforementioned to build models of the form \eqref{eq:LTI}, we introduce some relevant Koopman eigenfunction properties.
\begin{property}[\!\cite{Budisic2012}]\label{prop:productEFs}
	If the function space of eigenfunction is chosen to be a Banach algebra (e.g. $C^1(\set{X})$), the set of eigenfunctions forms an Abelian semigroup under point-wise products of functions. Thus, for $\mathcal{G}_{\ko_{\boldsymbol{f}}}$ with eigenvalues ${\lambda_{1}}$ and ${\lambda_{2}}$, $\phi_{1} \phi_{2}$ is also an eigenfunction of $\mathcal{G}_{\ko_{\boldsymbol{f}}}$ with eigenvalue ${\lambda_{1}+\lambda_{2}}$. 
\end{property}
{
{\begin{definition}[\!\cite{Mohr2014a,Folkestad2019}]\label{def:princEP}
Consider a system $\dot{\bm{y}}=\bm{A}\bm{y}$ and a multi-index $\bm{m}=[m_1,...,m_d]\in\mathbb{N}_0^d$ such that  $\|\bm{m}\|_1\tsgn{=}m_{1}\tsgn{+}\cdots\tsgn{+}m_{d}\leq {p}$ for $p \in \set{N}$.
	Consider $\{E^p\}$ to be the eigenpair group - a collection of all eigenvalue-eigenfunction pairs - of the Koopman operator generator $\mathcal{G}_{\ko_{\boldsymbol{A}}}$ for $\dot{\bm{y}}=\bm{A}\bm{y}$ with its minimal group generator $\mathcal{P}_{E}$:
	\begin{equation}\label{PeP}
	\{E^p\}=\left\{\left(\sum_{i=1}^{p} {{m_i\lambda_{i}}}, \prod_{i=1}^{p} \varphi_{i}^{m_i}\right) \mid\left(\lambda_{i}, \varphi_{i}\right) \subset \mathcal{P}_{E} \right\}.
	\end{equation}
 Then, the elements of $\mathcal{P}_{E}$ are \emph{principle} eigenvalues-eigenfunction pairs $\left(\lambda_{i}, \varphi_{i}\right)$ of $\mathcal{G}_{\ko_{\boldsymbol{A}}}$.
\end{definition}}
Less formally, \emph{principle eigenpairs} form the minimal set used to construct arbitrarily many other eigenpairs (\ref{PeP}).
}
\vspace{-0.5em}
\subsection{Topological Proxy to Koopman Eigenfunctions}
Here, we define the notions relevant for the geometric equivalence relations considered in this work.

\begin{definition}
	Consider a bijective map $\bm{g}: \mathbb{R}^{n} \mapsto \mathbb{R}^{n}$. The bijective map $\bm{g}$ is a \textit{homeomorphism} if both the map and its inverse $\bm{g}^{-1}$ are continuous. If the maps $\bm{g}$ and $\bm{g}^{-1}$ are also continuously differentiable, then $\bm{g}$ is a \textit{diffeomorphism}.
\end{definition}
\begin{definition}\label{def:conj}
	Two flows $\bm{F}^t: \set{X} \mapsto \set{X}$ and $\bm{C}^t: \set{Y} \mapsto \set{Y}$ of vector fields $\dot{\bm{x}}=\bm{f}(\bm{x})$ and $\dot{\bm{y}}=\bm{c}(\bm{y})$ are {topologically conjugate} if there exists a homeomorphism ${\bm{g}}: \set{X} \mapsto \set{Y}$ such that $\bm{g} \circ	\bm{F}^t = \bm{C}^t \circ \bm{g} $ holds  $\forall \bm{x} \in \set{X}$ and $t \in \set{R}$.
\end{definition}

\begin{proposition}[\!\cite{Budisic2012}]\label{specEQ}
	Consider the same two flows from Definition  \ref{def:conj}. If $(e^{\lambda t}, \varphi)$ is an eigenpair of $\mathcal{K}^t_{\bm{c}}$, then $(e^{\lambda t}, \varphi \circ \bm{g})$ is an eigenpair of $\mathcal{K}^t_{\bm{f}}$.
\end{proposition}

Extending the topological conjugacy to the entire region of attraction is formalized in the following.
\begin{proposition}[\!\cite{Folkestad2019}]\label{prop:l2nlEFs}
	Assume that the nonlinear system (\ref{sysKdiffeo}) is topologically conjugate to its Jacobian linearization via the diffeomorphism $\bm{d}: \set{X} \mapsto \set{Y}$. Let $\set{B} \subset \set{X}$ be a simply connected, bounded, positively invariant open set in $\set{X}$ such that $\bm{d}(\set{B}) \subset \set{Q}_{r} \subset \set{Y}$, where $\set{Q}_{r}$ is a cube in $\set{Y}$. Scaling $\set{Q}_{r}$ to the unit cube $\set{Q}_{1}$ via the diffeomorphism  $\bm{g}: \set{Q}_{r} \mapsto \set{Q}_{1}$ gives $(\bm{g} \circ \bm{d})(\set{B}) \subset \set{Q}_{1} .$ Then, if $\varphi$ is an eigenfunction for $\ko^t_{\bm{A}}$ at $e^{\lambda t}$, then $\varphi \circ \bm{g} \circ \bm{d}$ is an eigenfunction for $\ko^t_{\bm{f}}$ at eigenvalue $e^{\lambda t}$, where $\ko^{t}_{\bm{f}}$ is the Koopman operator semigroup associated with the nonlinear dynamics (\ref{sysKdiffeo}).
\end{proposition}
\begin{theorem}[\!\cite{Lan2013}]\label{thm:atMapLin}
	Consider the system (\ref{sysKdiffeo}) with $\bm{r}(\bm{x}) \in C^{2}(\set{X})$. Under Assumption \ref{ass:sysCLS}, the matrix $\bm{A}$ in \eqref{sysKdiffeo} is Hurwitz i.e., all eigenvalues have negative real parts. Then in the region of attraction $\set{S}$ of the origin there exists $\bm{\varrho}(\bm{x}) \in C^{1}(\set{S}): \set{S} \mapsto \set{R}^{d},$ such that $\bm{y}={\bm{d}}(\bm{x})=\bm{x}+\bm{\varrho}(\bm{x})$ is a $C^{1}$ diffeomorphism with $\bm{\varrho}(\bm{0})=\bm{0}$ in $\set{S}$ and satisfying $\dot{\bm{y}}=\bm{A} \bm{y}$.
\end{theorem}
By utilizing topological conjugacy, one exploits the fact that the eigenvalues are shared between the full nonlinear system (\ref{sysKdiffeo}) and its linearization around an asymptotically stable fixed point.
Subsequently, one is able to construct arbitrarily many eigenfunctions from the principal ones via Definition \ref{def:princEP} and Theorem \ref{specEQ}. Using these properties allows for learning Koopman-based dynamical models by lifting to eigenfunction coordinates {by design}.

\vspace{-0.5em}
\subsection{Asymptotic Stability Guarantees}
By Theorem \ref{thm:atMapLin} we deal with diffeomorphic mappings, motivating the definition of a stronger equivalence notion.
\begin{definition}[\!\cite{Meiss2007}] \label{def:smooth_equivalence}
	Vector fields $\dot{\bm{x}}\tsgn{=}\bm{f}(\bm{x})$ and $\dot{\bm{y}}\tsgn{=}\bm{t}(\bm{y})$ are said to be diffeomorphic, or smoothly equivalent, if there exists a diffeomorphism $\bm{d}: \mathbb{R}^{d} \mapsto \mathbb{R}^{d}$ such that
	$\forall \bm{x} \in \mathbb{R}^{d}$ $\bm{t}(\bm{d}(\bm{x}))\tsgn{=}\bm{J}_{\bm{d}}(\bm{x}) \bm{f}(\bm{x})$ holds.
\end{definition}
To ensure safety, we are interested in transferring the asymptotic stability properties of the linearization to the lifted linear system \eqref{eq:LTI}. That is enabled by the following result.
{{
\begin{theorem}\label{thm:diffNL2LTIgas}
		Consider a system \eqref{sysKdiffeo} satisfying Assumption \ref{ass:sysCLS}, its Jacobian linearization $\dot{\bm{y}}=\bm{A}\bm{y}$ and a diffeomorphic map $\bm{d}$. Let the Koopman eigenfunctions \eqref{eq:cLTI:1} corresponding to $\mathcal{G}_{\ko_{\bm{f}}}$ of system \eqref{sysKdiffeo} be constructed via Proposition \ref{prop:l2nlEFs} utilizing Definition \ref{def:princEP}. Then, the associated $\mathcal{G}_{\ko_{\bm{f}}}$-realization of the form \eqref{eq:LTI} is guaranteed to be asymptotically stable.
		
	\begin{proof}\label{thmproof}
{First we show the transition matrix in \eqref{eq:cLTI:2} is Hurwitz. As $\bm{A}$ is Hurwitz by Assumption \ref{ass:sysCLS}, then $\bm{\Lambda}$ in \eqref{eq:cLTI:2} is as well with eigenvalues satisfying $\operatorname{Re}[\sum^{p}_{i=1}m_i{\lambda}_{i}] < 0$ per Definition \ref{def:princEP}.
Secondly, we show the lifting \eqref{eq:cLTI:1} is an immersion - $\operatorname{rank}(\bm{J}_{\bm{\phi}}(\bm{x}))\tsgn{=}\operatorname{dim}(\bm{x})$. Consider an eigenfunction library $\bm{\varphi}(\bm{y})$ constructed by concatenating elements of $\{E^p\}$ of Definition \ref{def:princEP}. As it forms a monomial basis, its rank equals $\operatorname{dim}(\bm{y})$ making it an immersion.
Diffeomorphisms $\bm{g}$ and $\bm{d}$ prescribed by Proposition \ref{prop:l2nlEFs} are immersions by definition. As compositions of immersions are an immersion, the immersibility of $\bm{\phi}\tsgn{=}\bm{\varphi}\tsgn{\circ}\bm{g}\tsgn{\circ}\bm{d}$ is ensured. With $\bm{\Lambda}$ Hurwitz and $\bm{\phi}$ immersible, the asymptotic stability of the lifted model \eqref{eq:LTI} follows via \cite[Proposition 1]{Yi2021}.}
	\end{proof}	
\end{theorem}
\begin{remark}
The result of Theorem \ref{thm:diffNL2LTIgas} establishes that the asymptotic stability of the Jacobian linearization carries over to the linear predictor of the form \eqref{eq:LTI} when constructed by diffeomorphically transforming the eigenfunctions of a linear system.
\end{remark}
}
}
\vspace{-0.0em}
\section{Learning Equivalences via Invertible Neural Networks}\label{sec:KEFlows}
In order to learn the Koopman eigenfunctions through an equivalence relation for the system (\ref{sysKdiffeo}) in a well-conditioned manner, one needs to ensure the function approximator is constrained to be a diffeomorphism.
Allowing for that are flow-based neural networks, where coupling flow invertible neural networks (CF-INN) present a powerful tool.
Although with their form restricted compared to vanilla neural networks, there are CF-INN architectures exhibiting $L^{p}$-universality/$\operatorname{sup}$-universality for a large class of diffeomorphisms \cite{Teshima2020}. Hence, they can be relied on for learning of Koopman eigenfunctions via Proposition \ref{prop:l2nlEFs}.. 
\begin{assumption}\label{ass:univDiffeo}
	Let $D^2(\set{X})$ be the set of all $C^2(\set{X})$ diffeomorphisms. Then, there is a class $\mathcal{D}$ of universal approximators $\bm{{\hat{d}}}$ such that for any ${\bm{d}} \in{D}^2(\set{X})$ and any $\varepsilon > 0, \exists \bm{{\hat{d}}} \in \mathcal{D}$ such that $\| \bm{d}-\bm{{\hat{d}}} \|_{p, \set{X}}<\varepsilon$.
\end{assumption}
The above assumption is hardly restrictive as it is fulfilled by almost all diffeomorphisms and thus systems \eqref{sysKdiffeo}. Therefore, we can state the following result on the expressivity of learning Koopman eigenfunctions constructed via Proposition \ref{prop:l2nlEFs}.
\begin{lemma}\label{lem:uniEFs}
  	Let there exist a $C^2$-diffeomorphism making the system \eqref{sysKdiffeo} smoothly equivalent to its Jacobian linearization. Consider a function approximator $\bm{{\hat{d}}}$ fulfilling Assumption \ref{ass:univDiffeo} and let $\mathcal{G}_{\ko_{\bm{f}}}$-eigenfunctions of \eqref{sysKdiffeo} be constructed via Proposition \ref{prop:l2nlEFs} such that an approximate eigenfunction has the form $\hat{\phi} = \varphi \circ \bm{g} \circ \bm{{\hat{d}}}$. Then, for any $\delta > 0$ and any ${\phi}$,  $\exists \hat{\phi}$ such that $\|{\phi} - \hat{\phi}\|_{p, \set{X}} < \delta$.
  	
 \begin{proof}
    Due to continuity on their corresponding domains, there exist Lipschitz constants $L_{\varphi}$ and $L_{\bm{g}}$, associated to $\varphi$ and $\bm{g}$ respectively so that $\|\bm{d}-\bm{{\hat{d}}} \|_{p, \set{X}}<\frac{\delta}{L_{\varphi}L_{\bm{g}}}$ holds. Then, we have
 	\begin{equation}
 	\begin{aligned}
 	\|{\phi} - \hat{\phi}\|_{p, \set{X}} & \stackrel{\text{Prop. \ref{prop:l2nlEFs}}}{=} \|\varphi \circ \bm{g} \circ {\bm{d}} - \varphi \circ \bm{g} \circ \bm{{\hat{d}}}\|_{p, \set{X}} \\
 	& \quad \leq \quad L_{\varphi}L_{\bm{g}} \|{\bm{d}} - \bm{{\hat{d}}}\|_{p, \set{X}} < \delta,
 	\end{aligned}
 	\end{equation}
	concluding the proof.
 \end{proof}
\end{lemma}

\vspace{-0.5em}
\subsection{Affine Coupling Flows (ACF)}
\label{sec:affine_coupling_flows}
The approach of NF is to compose a complicated bijective function successively from multiple simpler bijections - using the fact that the composition of bijective functions is again bijective. In our case, we form a diffeomorphism between the nonlinear system and its linearization by multiple simpler diffeomorphisms $\bm{{\hat{d}}_i}$ so that $\bm{y}=\bm{{\hat{d}}}(\bm{x}) = \bm{{\hat{d}}_k} \circ ... \circ \bm{{\hat{d}}_1}(\bm{x})$ - as visualized in Fig. \ref{fig:ACF}.
\begin{figure}[htb]
    \centering
    \includegraphics[width=1.\linewidth]{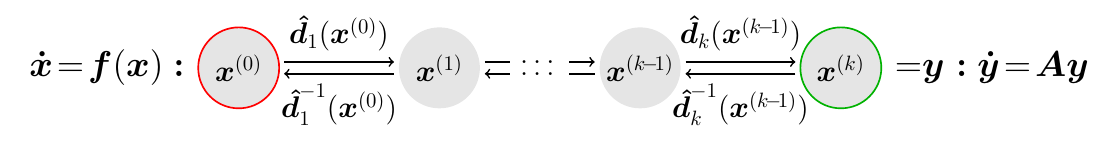}
    \caption{Construction of a linearizing diffeomorphism}
    \label{fig:ACF}
\end{figure}
The bijectivity of the individual functions $\bm{{\hat{d}}_i}$ is ensured by a special structure - called affine coupling layers
\begin{equation}
\resizebox{.8\linewidth}{!}{
$\bm{{\hat{d}}}_{i}(\bm{x}^{(i)}) = \begin{bmatrix}\bm{x_a}^{(i)}\\\bm{x_b}^{(i)} \odot \operatorname{exp}(\bm{s_i}(\bm{x_a}^{(i)})) + \bm{t_i}(\bm{x_a}^{(i)})\end{bmatrix}$,
}
\label{eq:coupling_layer} 
\end{equation}

where $\bm{x_a} \equiv x_1,...,x_n$, $\bm{x_b} \equiv x_{n+1},...,x_N$, $\odot$ denotes Hadamard product and $\operatorname{exp}$ denotes pointwise exponential. The input vector $\bm{x}^{(i)}$ is split dimension-wise into two parts $\bm{x}_{a}^{(i)}$ and $\bm{x}_{b}^{(i)}$; $\bm{x}_{a}^{(i)}$ is then scaled with $\operatorname{exp}(\bm{s}(\bm{x}))$, and translated with $\bm{t}(\bm{x})$, and multiplied/added element-wise to $\bm{x}_{b}^{(i)}$. It is important to change the policy of splitting for each affine coupling layer to leave any component unaltered. The scaling functions $\bm{s}_i:\mathbb{R}^n \mapsto \mathbb{R}^{N-n}$ and translation functions $\bm{t}_i:\mathbb{R}^n \mapsto \mathbb{R}^{N-n}$ are learned and parametrized by multi-layer neural networks with suitably smooth activation functions to form a diffeomorphism. The parameters that have to be trained in order to learn a diffeomorphisms are therefore the weights and biases in the neural networks of the scaling and translation functions - concatenated in parameters $\bm{w}=[\bm{w}^{\top}_{\bm{s}_1},\bm{w}^{\top}_{\bm{t}_1}, \cdots, \bm{w}^{\top}_{\bm{s}_k},\bm{w}^{\top}_{\bm{t}_k}]^{\top}$. 
\begin{remark}
As the special structure of the affine coupling layer already ensures bijectivity and invertibility by design, there are no restrictions to the neural networks of the scaling and translation functions and they can be learned freely. 
\end{remark}

{\subsection{Supervised Learning of a Linearizing Diffeomorphism}}\label{subs:ACFdesign}
\label{sec:cost_function}
As a corollary of Theorem~\ref{thm:atMapLin} and Definition~\ref{def:smooth_equivalence} the following equations
\begin{equation}\label{eq:diffeo_learning_topological_conjugacy} 
\bm{{\dot{x}}} {=} \bm{J_{{{d}}}}^{-1}(\bm{x})\bm{A} \bm{{{d}}}(\bm{x}),~\bm{J_{d}}(\bm{0}) {=}\bm{I},~\bm{d}(\bm{0}){=}\bm{0},  \end{equation}
are to be fulfilled by a linearizing diffeomorphism.
Assume the availability of a data-set of $N$ input-output pairs $\mathbb{D}_{N}=\left\{ \dot{\bm{x}}^{(i)},\bm{x}^{(i)}\right\}_{i=1}^{N}$ for the system \eqref{sysKdiffeo} satisfying Assumption \ref{ass:sysCLS}. Then, the solution $\bm{{\hat{d}}}(\bm{x}):=\bm{{{d}}_{\bm{{\hat{w}}}}}(\bm{x})$ to (\ref{eq:diffeo_learning_topological_conjugacy}) can be obtained in terms of the ACF parameters $\bm{w}$ by solving the following optimization problem
\begin{subequations}\label{eq:loss_function} 
\begin{align}
\bm{{\hat{w}}} = \argmin_{\bm{w}} \sum_{i=1}^{N} &\|\bm{{\dot{x}}} \tsgn{-}\bm{J_{{{\bm{{d_{w}}}}}}}^{-1}(\bm{x})\bm{A} \bm{{{\bm{{d_{{w}}}}}}}(\bm{x})\|_2^2\tsgn{+}\\&\|\bm{J_{{\bm{{d_{{}}}}}}}(\bm{0})\tsgn{-}\bm{I}\|_2^2\tsgn{+}\|\bm{d_{}}(\bm{0})\tsgn{-}\bm{0}\|_2^2.
\end{align}
\end{subequations}
\subsection{Constructing Nonlinear System's Eigenfunctions}
\label{sec:construct_ef}
Algorithm~\ref{alg:KEFs} provides a pseudo code for the construction of Koopman eigenfunctions. Since the eigenfunctions are constructed via the learned diffeomorphism through Normalizing Flows, we call this approach \textit{KoopmanEigenFlows}. 
We first calculate the eigenfunctions of the linearized system. Let $\bm{v}=\{\bm{v_1},...,\bm{v_d}\}$ be the eigenvector-basis of matrix $\bm{A}$ corresponding to non-zero eigenvalues $\{\lambda_1,...,\lambda_d\}$. Then the adjoint basis $\bm{w}=\{\bm{w_1},...,\bm{w_d}\}$ is given via the transposed cofactor-matrix of $\bm{v}$ so that $\bm{\langle v_i,w_j\rangle} \tsgn{=} \delta_{ij}$ and $\bm{w_k}$ is an eigenvector of $\bm{A^*}$ at eigenvalue $\overline{\lambda}_k.$
Then, the inner product $\varphi_{\text{p},i} = \bm{\langle y,w_j \rangle}$ is a nonzero principal eigenfunction of the Koopman generator $\mathcal{G}_{\mathcal{K}_{\bm{A}}}$ of the linearized system cf. \cite[Prop. 1]{Folkestad2019}. As a corollary of~\eqref{evoOF} for observable $\bm{h}=\operatorname{id}$, we know that the modal decomposition of $\bm{x}$ is given through
\begin{equation}
\bm{{{x}}} = \sum_{j=1}^{\infty} \bm{v}_j{\phi}_j(\bm{x}) \ , \label{eq:reconstruction_of_x_infinite_sum} 
\end{equation} 
where the infinite sum results from the infinite dimensionality of the Koopman operator. 
For a practicable representation of \eqref{eq:reconstruction_of_x_infinite_sum}, we create a library of eigenfunctions taking the principle ones to predefined maximum powers $p^{(1)},...,p^{(d)}$ for each principal eigenfunction $\varphi_{\text{p},i}$ using Property \ref{prop:productEFs}.
With this finite number of eigenfunctions, \eqref{eq:reconstruction_of_x_infinite_sum} can be written in matrix-vector notation, i.e. $\bm{{\hat{x}}}=\bm{V}{\bm{{\hat{\phi}}}}(\bm{x})=\bm{V}\bm{z}$, with the reconstruction matrix $\bm{V}$. The simple library of eigenfunctions of the linearized system $\bm{{{\varphi}}}$ is then converted to a library of eigenfunctions of the nonlinear system via $\bm{{\hat{\phi}}}(\bm{x})=\bm{\varphi} \circ \bm{g} \circ \bm{{\hat{d}}}(\bm{x})$. The eigenvalues are preserved, since they are shared between topologically conjugate systems.

\begin{algorithm}
\caption{\small KoopmanEigenFlows}\label{alg:KEFs}
\begin{algorithmic}[1]
    \Statex {\small \textbf{Input:} Jacobian linearization $\bm{A}$\textbf{;} $\mathbb{D}_{N}\tsgn{=}\{ {\bm{{\dot{x}}}}^{(i)},\bm{x}^{(i)}\}^{N}_{i=1}$\textbf{;}  maximum powers $p^{(1)},...,p^{(d)}$ of the $d$ principal eigenpairs}
    \State {\small Learn a diffeomorphism $\bm{{\hat{d}}}$ through NF:}
    \Statex {\small $\bm{{\hat{d}}}(\bm{x}) \gets NF (\textbf{id}(\bm{x}))$}
    \Statex {\small $\bm{{\hat{d}}}(\bm{x}) \gets \eqref{eq:loss_function}$}
    \State {\small Construct principal eigenpairs of the linearized system:}
    \Statex {\small $\lambda_{\text{p},j}, \varphi_{\text{p},j} \gets \bm{A}, \langle y,w_j \rangle$,~~~ $j=1,...,d$}
    {
    \State {\small Construct library of eigenpairs of the linearized system:}
    \Statex {\small \textbf{for} $M=$ all combinations of $[p^{(1)},...,p^{(d)}]$ \textbf{do}}
    \Statex {\small $~~~\lambda_i \gets \sum_{j=1}^d  {\bm{M}[i,j] \cdot \lambda_{\text{p},j}}$,~~~$i=1,...,\prod_{k=1}^{d}p^{(k)}+1$}
    \Statex {\small $~~~\varphi_{i} \gets \prod_{j=1}^d  \varphi_{\text{p},j}^{\bm{M}[i,j]}$}
    }
    \State {\small Construct eigenpairs of the nonlinear system:}
    \Statex {\small ${\displaystyle \bm{{\hat{\phi}}} \gets \bm{\varphi} \left(\bm{g}(\bm{{\hat{d}}}(\bm{x}))\right), ~~~~~~~~\bm{\varphi}=[\varphi_{1},...,\varphi_{i},...,\varphi_{D}]^T}$}
    \Statex {\small \textbf{Output:} A library of eigenvalue-eigenfunction pairs $(\bm{\lambda}, \bm{{\hat{\phi}}})$}
\end{algorithmic}
\end{algorithm}
\begin{corollary}\label{cor:GASanyway}
Let diffeomorphisms $\bm{d}=\bm{{\hat{d}}_k} \circ \ldots \circ \,\bm{{\hat{d}}_1}(\bm{x})$ be parameterized through coupling layers  \eqref{eq:coupling_layer}, which are defined using $C^1$-functions $\bm{s}_i$, $\bm{t}_i$.
Then, every solution of the optimization problem \eqref{eq:loss_function} constructed via Proposition \ref{prop:l2nlEFs} utilizing Definition \ref{def:princEP} yields a stable system \eqref{eq:LTI}.\looseness=-1
\end{corollary}
\begin{proof}
As ACFs form diffeomorphisms per construction with $C^1$-function approximators $\bm{s}_i$ and $\bm{t}_i$, the asymptotic stability asserted by Theorem \ref{thm:diffNL2LTIgas} holds regardless of the approximation accuracy.
\end{proof}
\begin{remark}
In essence, the result of Cor. \ref{cor:GASanyway} decouples safety from performance in Alg. \ref{alg:KEFs} as asymptotic stability is guaranteed regardless of how well \eqref{eq:loss_function} fits \eqref{eq:diffeo_learning_topological_conjugacy}.
\end{remark}
\section{{KoopmanEigenFlow} Mode Decomposition (KEFMD)}\label{sec:KEFDMD}
To construct linear predictors using Koopman eigenfunctions for nonlinear dynamics, we develop a method to build a linear model in the space of Koopman eigenfunction-observables. As \textit{{KoopmanEigenFlows}} are utilized to construct the Koopman eigenfunction coordinates, we name our algorithm \textbf{K}oopman\textbf{E}igen\textbf{F}low \textbf{M}ode \textbf{D}ecomposition (KEFMD). 
Since we inherit the spectrum from the linearization we do not append ``Dynamic Mode Decomposition". Furthermore, mode decomposition is a general concept resulting from operator theoretic identification - transcending the original DMD  algorithm \cite{schmid_2010} that only considers observables that are linear functions of the state.
\vspace{-0.3em}
\subsection{The KEFMD Framework}
{The developments of Sec. \ref{sec:sysTHRY} and \ref{sec:KEFlows} allow for a system identification approach that extends the spectral properties of a Jacobian linearization of the system around an equilibrium to the corresponding nonlinear system.}
With the output of Algorithm~\ref{alg:KEFs}, an LTI-system in eigenfunction coordinates - as defined in \eqref{eq:LTI} - is constructed. Once the initial condition is lifted \eqref{eq:cLTI:1}, the evolution \eqref{eq:cLTI:2} as well as the state reconstruction \eqref{eq:cLTI:3} is linear - permitted through Assumption \ref{ass:sysCLS}. The corresponding flow of the original state is given explicitly as $\bm{{\hat{x}}}(t) = \bm{V}e^{\bm{\Lambda}t}\bm{z}_0 \  \label{eq:flow_x(t)}$. We can define a discretized matrix $\bm{\Lambda}_d=e^{\bm{\Lambda}t}$ by fixing $t$ to a certain time-step value. This way, we obtain a discretized linear evolving model where
$\bm{{\hat{x}}}^{k+1} = \bm{V}\bm{\Lambda}_d^k \bm{z}_0 .\  \label{eq:x_k}$
with the discrete-time index $k$, so that the state $\bm{x}$ can be predicted linearly. The pseudocode is provided in Algorithm~\ref{alg:LTI}.

\begin{algorithm}
\caption{\small KoopmanEigenFlow Mode Decomposition$\!-\!$KEFMD}\label{alg:LTI}
\begin{algorithmic}[1]
    \Statex {\small \textbf{Input:} Eigenvalue-eigenfunction pairs $(\bm{\lambda}, \bm{{\hat{\phi}}})$ from Alg. \ref{alg:KEFs}  }
    \State {\small Construct the \textit{lifted} LTI-system:}
    \Statex {\small $\bm{\Lambda} \gets \operatorname{diag}(\bm{\lambda})$}
    \Statex {\small $\bm{z} \gets \bm{{\hat{\phi}}} (\bm{x})$}
    \Statex {\small $\bm{V} \gets \bm{x}\bm{z}^\dagger$}
    \State {\small $\bm{{\dot{{\hat{x}}}}} \gets \bm{V \Lambda z}$}
    \State {\small $\bm{{\hat{x}}}^{k+1} \gets \bm{V} \bm{\Lambda}_d^k \bm{z}_0$}
    \Statex {\small \textbf{Output:} $\bm{\Lambda},\bm{\Lambda}_d,\bm{V},\bm{{\hat{\phi}}}$}
\end{algorithmic}%
\end{algorithm}

\vspace{-1em}
\section{Evaluation}\label{sec:simres}
To validate the proposed approach, we demonstrate the performance of the constructed Koopman operator dynamical models on two examples with different properties.
In both, ACF with 7 coupling layers are used, whose composition results in the diffeomorphism. The neural networks for the scaling and translation functions in each affine coupling layers have 3 hidden layers, 120 neurons each with an Exponential Linear Unit (ELU) as the activation function. Batch learning is performed with a batch size of 64 and the initialization of the weights is defined such that $\boldsymbol{{\hat{d}}}(\boldsymbol{x})$ is an identity map.

\begin{example}[A simple illustrative example]\label{ex1}
Consider the following dynamical system \cite{Folkestad2019}:
\begin{equation}\label{SMdt}
\bm{{\dot{x}}}=\left[\!\begin{array}{c}
\mu x_{1} \\
\lambda\left(x_{2}-x_{1}^{2}\right)
\end{array}\!\right]=
\left[\begin{array}{cc}
	\mu & 0 \\
	0 & \lambda
	\end{array}\right]
	\boldsymbol{x}+
	\left[\begin{array}{c}
	0 \\
	\lambda\left(-x_{1}^{2}\right)
	\end{array}\right],
	\end{equation}
with $\mu\tsgn{=}\tsgn{-}0.7$ and $\lambda\tsgn{=}\tsgn{-}0.3$. $N\tsgn{=}4800$ training data is generated from $24$ equally long trajectories with sampling time $dt\tsgn{=}0.065s$, with the starting points are uniformly distributed on the edges of $\set{X}\tsgn{=}[\tsgn{-}5,5]^{2}$. The maximum powers are set to $p^{(i)}\tsgn{=}5$ which results in $36$ lifted coordinates $\boldsymbol{z}$. For this system the exact diffeomorphism $\boldsymbol{d}(\bm{x})\tsgn{=}[x_1,x_2\tsgn{-}\frac{\lambda}{\lambda-2 \mu}x_1^2]^{\top}$ is known, allowing us to directly compare it to the learned one. Fig.~\ref{fig:motivating_example_diffeo} shows the error between the learned and the true diffeomorphism; demonstrating the mapping is well captured by our ACF design from Subs. \ref{subs:ACFdesign}.

\begin{figure}[htb]
\centering
\includegraphics[width=.99\linewidth]{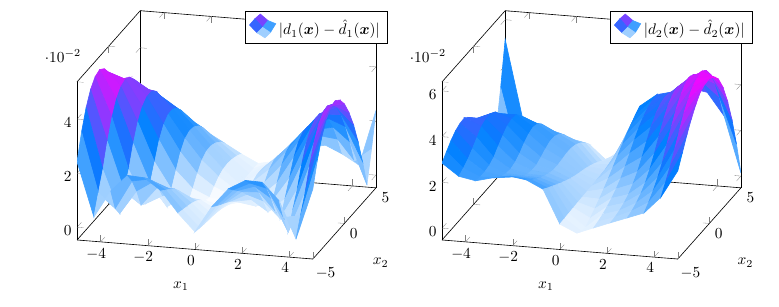}
\caption{\small Dimension-wise plot of the error between the true and learned diffeomorphism.}
\label{fig:motivating_example_diffeo}
\end{figure}

The main goal, however, is the linear prediction of the state $\bm{x}$ - given through $\boldsymbol{{\hat{x}}}^{k+1} \tsgn{=} \boldsymbol{V}\boldsymbol{\Lambda}_d^k \boldsymbol{z}_0$ in \textit{KEFMD}. To validate our algorithm we compare it with the related KEEDMD~\cite{Folkestad2019} and the established EDMD approaches with monomials and radial basis functions - in their continuous-time, Koopman generator versions \cite{Klus2020}.
In contrast to our deep learning approach, the EDMD approaches use a predefined, function basis as lifting functions. A basis that is not data-driven provides good predictive performance only when it spans a Koopman-invariant subspace. As Koopman-invariant coordinates are unsupervised features - presuming a suitable basis is available \textit{apriori} is a strong assumption. The KEEDMD~\cite{Folkestad2019} approach has similar system-theoretic considerations, but employs a very different learning procedure.

Tab.~\ref{table_comparison} shows an statistical evaluation (mean and standard deviation) of the root mean squared errors (RMSE's) of $100$ trajectories with start points on an uniformly spaced grid for each of the methods. The suitable size of the lifting dimension can be determined via cross-validation. Besides KEFMD, the EDMD-monomials approach gives an almost perfect result since the relevant eigenfunctions for the system of Ex. \ref{ex1} are spanned by monomials.
\end{example}




\begin{example}[Exact modes unknown \textit{apriori}]\label{ex3}
	Consider the following dynamical system:
	\begin{equation}
    \boldsymbol{{\dot{x}}}=\boldsymbol{f}(\boldsymbol{x}) = \begin{bmatrix}\left(a+c \cdot \operatorname{sin}^{2}(x_{2})\right)x_{1}\\b x_{2}\end{bmatrix},\
    \label{eq:ex3}
    \end{equation} 
with $a\tsgn{=}\tsgn{-}1.3$, $b\tsgn{=}\tsgn{-}2$, $c\tsgn{=}1.5$ and the Jacobian linearization $\dot{\bm{y}}\tsgn{=}[a y_{1} \quad b y_{2}]^{\top}$.
The data set with $N\tsgn{=}11200$ data points is generated from $56$ equally long trajectories with sampling time $dt\tsgn{=}0.015s$, starting on the edges of $\set{X}\tsgn{=}[\tsgn{-}5.5,5.5]^2$. The number of lifted coordinates is $196$. For this system we do not know the explicit diffeomorphism and do not \textit{apriori} know whether an exact finite representation exists. 
\end{example}
Fig.~\ref{fig:trajectories_keedmd} shows how our learned Koopman operator dynamical model, although linearly evolving, reproduces a clearly nonlinear time response in original coordinates. Therein are also the resulting state trajectories generated by the different approaches with identical training data. Only KEFMD captures the nonlinear system's behavior well - as quantified in Tab.~\ref{table_comparison} and illustrated in Fig. \ref{fig:trajectories_keedmd}. Note that the differing dimension of EDMD-monomials in Ex. \ref{ex3} is due to higher order monomials not performing as well. 
{\begin{remark}[Ill-posedness of KEEDMD]
The related approach of KEEDMD \cite{Folkestad2019} does not pose a fully supervised learning problem due to not having explicit training targets as our approach does in \eqref{eq:loss_function}. The aforementioned, coupled with the use of vanilla NNs - a non-diffeomorphic hypothesis class - results in a severely underdetermined formulation. With the addition of employing heuristics that are not theoretically justified, the KEEDMD learning framework is ill-posed - which the performance evaluation in Tab. \ref{table_comparison} demonstrates as well. 
Even though considering stable systems, it offers no stability guarantees - making the ill-posedness also apparent in Fig. \ref{fig:trajectories_keedmd} - as the trajectories show no convergence to the origin.
\end{remark}}
\vspace{-0.5em}
\setlength\tabcolsep{1pt} 
\begin{table}[ht]
    \centering
    \resizebox{1.0\linewidth}{!}{
    \begin{tabular}[t]{l||l|l|l|l}
    $ $ & {\normalsize \textbf{KEFMD}} & {\normalsize KEEDMD \tablefootnote{\hyperlink{https://github.com/Cafolkes/keedmd}{https://github.com/Cafolkes/keedmd}\label{gitkeedmd}}} & {\normalsize EDMD-mon.\tablefootnote{\hyperlink{https://github.com/sklus/d3s}{https://github.com/sklus/d3s}}} & {\normalsize EDMD-RBF $^{\text{\ref{gitkeedmd}}}$} \\
    \hline
    \hline
    {\normalsize \ref{ex1}} & ${\normalsize \textbf{0.002}\pm 0.001}^{(36)}$ & ${\normalsize 9.12 \pm 6.92}^{(36)}$ & ${\normalsize \textbf{0.001} \pm 0.001}^{(36)}$ & ${\normalsize 1.19 \pm 0.52}^{(36)}$ \\
    \hline
    {\normalsize \ref{ex3}} & ${\normalsize \textbf{0.014} \pm 0.000}^{(196)}$ & ${\normalsize 1.41 \pm 0.62}^{(196)}$ & ${\normalsize 0.71 \pm 0.8}^{(81)}$ & ${\normalsize 0.76 \pm 0.32}^{(196)}$ \\
    \end{tabular}
    }
 \caption{\small RMSE mean and standard deviation for different approaches. The lifted state dimension is written in brackets.}
    \label{table_comparison}
\end{table}
\begin{figure}[htb]
\centering
\includegraphics[width=.99\linewidth]{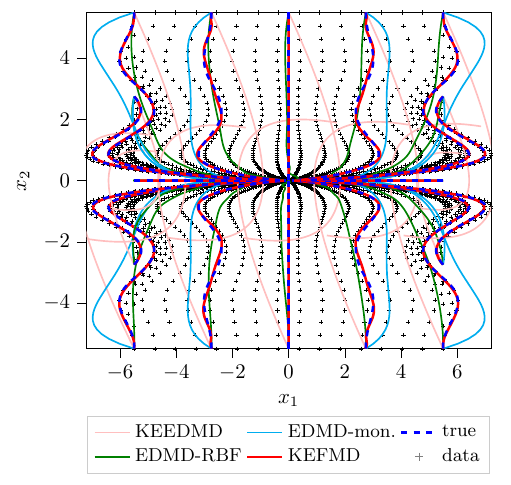}
\caption{\small Linear prediction performance of the different approaches.\protect\footnotemark\ }
\label{fig:trajectories_keedmd}
\end{figure}
\footnotetext{For clarity of presentation every 5th datapoint is plotted.}

\section{CONCLUSION}\label{sec:cnclsn}

In this paper, we present a novel reliable and safe framework for learning Koopman eigenfunctions for constructing linear prediction models for a class of nonlinear dynamics. 
These results demonstrate superior performance compared to related works and showcase the utility and transferrability of Koopman operator theory to data-driven realizations.
The reliable learning of our approach offers extensions to controlled systems for efficient optimal control using linear systems theory.

\begin{spacing}{0.9}
\bibliographystyle{IEEEtran}
\bibliography{bibfile}
\end{spacing}


\end{document}